\newcommand{\shortversion}[1]{}
\newcommand{\longversion}[1]{#1}

\shortversion{
\documentclass[letterpaper]{article}
\usepackage{ijcai13}
}
\longversion{
\documentclass[letterpaper]{article}
\usepackage{named,fixbib}
\usepackage[totalwidth=450pt,totalheight=640pt]{geometry}
}

\usepackage{times}

\usepackage[scaled=0.92]{helvet}
\usepackage{amsmath}
\usepackage{amssymb}
\usepackage{amsthm}
\usepackage{ifthen}
\usepackage{xspace}
\usepackage{colonequals}
\usepackage[donotfixamsmathbugs]{mathtools}
\usepackage[linesnumbered,vlined,ruled]{algorithm2e}
\usepackage{paralist}

\shortversion{\newcommand{\citex}[1]{\citeauthor{#1}~\shortcite{#1}}}
\longversion{\newcommand{\citex}[1]{\citeauthor{#1}~\shortcite{#1}}}

\newcommand{\Nat}{\mathbb{N}}

\newcommand{\CCard}[1]{\|#1\|}

\theoremstyle{plain}
\newtheorem{theorem}{Theorem}

\newtheorem{lemma}[theorem]{Lemma}
\newtheorem{corollary}[theorem]{Corollary}
\newtheorem*{theorem*}{Theorem}
\newtheorem*{fact*}{Fact}
\theoremstyle{definition}
\newtheorem{definition}[theorem]{Definition}
\newtheorem*{definition*}{Definition}
\theoremstyle{remark}

\newcommand {\bigO} {{\mathcal O}}

\newcommand {\card} [1] {\left | #1 \right |}

\newcommand {\ra} {\rightarrow}
\newcommand {\abdinst} [1] {\ensuremath{\langle}#1\ensuremath{\rangle}\xspace}

\newcommand {\cB} {\mathcal{B}}
\newcommand {\cC} {\mathcal{C}}

\newcommand {\cP} {\mathcal{P}}

\newcommand {\horn}      {\textsc{Horn}\xspace}

\newcommand {\krom}      {\textsc{Krom}\xspace}
\newcommand {\cnf}      {\textsc{Cnf}\xspace}

\newcommand {\sol} {\mathit{Sol}}
\newcommand {\var}[1]{\ensuremath{\textsl{var}\,(#1)}}
\renewcommand {\Res} {\mathit{Res}}
\newcommand {\TrimRes} {\mathit{TrimRes}}

\newcommand{\ccfont}[1]{\textsf{#1}}
\newcommand{\NP}{\textnormal{\ccfont{NP}}\xspace}
\newcommand{\coNP}{\textnormal{\ccfont{co-NP}}\xspace}

\newcommand{\Ptime}{\textnormal{\ccfont{P}}\xspace}
\newcommand{\SigmaTwo}{\ccfont{\ensuremath{\mathsf{\Sigma}_2^\Ptime}}\xspace}

\newcommand{\paraNP}{\textsf{para-NP}\xspace}

\newcommand{\abd}{\textsc{Abd}\xspace}
\newcommand{\sat}{\textsc{Sat}\xspace}
\newcommand{\qbf}{\textsc{Qbf}\xspace}
\newcommand{\abdmin}{\textsc{Abd}$_\subseteq$\xspace}

\newcommand{\ta}[1]{\text{\normalfont \textsf{ta}(\ensuremath{#1})}}

\newcommand{\true}{\textit{true}\xspace}
\newcommand{\false}{\textit{false}\xspace}

\newcommand{\TR}{\mathit{TR}}
\shortversion{
\newcommand{\parhead}[1]{\smallskip\noindent\textbf{#1.}\ }
}
\longversion{
\newcommand{\parhead}[1]{\paragraph{#1.}}
}

\makeatletter
\longversion{
\usepackage{etoolbox}
\newlength{\algorithmwidth}
\newlength{\algorithmremainder}
\setlength{\algorithmwidth}{0.75\textwidth}
\setlength{\algorithmremainder}{\textwidth-\algorithmwidth}
\patchcmd{\@algocf@start}{%
  \begin{lrbox}{\algocf@algobox}%
}{%
  \rule{0.5\algorithmremainder}{\z@}%
  \begin{lrbox}{\algocf@algobox}%
  \begin{minipage}{\algorithmwidth}%
}{}{}
\patchcmd{\@algocf@finish}{%
  \end{lrbox}%
}{%
  \end{minipage}%
  \end{lrbox}%
}{}{}}
\makeatother

\shortversion{

\renewenvironment{proof}{\vspace{-3mm}\begin{pf}}{\qed\end{pf}}
}

\allowdisplaybreaks
\shortversion{\frenchspacing}
\title{Backdoors to Abduction}
\shortversion{ 
  \author{%
    Andreas  Pfandler\thanks{Supported by the Austrian Science Fund (FWF):
      P25518-N23.} 
    \qquad
    Stefan R\"ummele\footnotemark[1] 
    \qquad
    Stefan  Szeider\thanks{Supported by the European Research Council (ERC),
      project 239962.}\\[4pt]
 Vienna University of Technology, Austria\\
 \{pfandler, ruemmele\}@dbai.tuwien.ac.at, stefan@szeider.net
} 
}
\longversion{ 
  \author{%
    Andreas  Pfandler\thanks{Supported by the Austrian Science Fund (FWF):
      P25518-N23.} \qquad
    Stefan R\"ummele\footnotemark[1] 
    \qquad
    Stefan  Szeider\thanks{Supported by the European Research Council (ERC),
      project 239962.}\\[4pt]
\small Vienna University of Technology, Austria\\[-2pt]
\small \{pfandler, ruemmele\}@dbai.tuwien.ac.at, stefan@szeider.net
} 
\date{}
}

\begin{document}
\maketitle
\longversion{\thispagestyle{empty}}

\begin{abstract}
 \longversion{\noindent}Abductive reasoning (or Abduction, for short) is among the most
  fundamental AI reasoning methods, with a broad range of
  applications, including fault diagnosis, belief revision, and
  automated planning.  Unfortunately, Abduction is of high
  computational complexity; even propositional Abduction is
  \SigmaTwo-complete and thus harder than \NP and \coNP.  This
  complexity barrier rules out the existence of a polynomial
  transformation to propositional satisfiability (\sat).  In this work
  we use structural properties of the Abduction instance to break this
  complexity barrier. We utilize the problem structure in terms of
  small backdoor sets.  We present fixed-parameter tractable
  transformations from Abduction to \sat, which make the power of today's \sat solvers available to Abduction.%
\end{abstract}

\section{Introduction}

Abductive reasoning (or \emph{Abduction}, for short) is among the most
fundamental reasoning methods.  It is used to explain observations by
finding appropriate causes.  In contrast to deductive reasoning, it is
therefore a method for ``reverse inference''.
Abduction has a broad
range of applications in AI, including system and medical
diagnosis, planning, configuration, and database
updates\longversion{~\cite{BylanderATJ91,NGM92,Pople73}}.

Unfortunately, Abduction is of high computational complexity. Already
propositional Abduction, the focus of this paper, is
\SigmaTwo-complete and thus located at the second level of the
Polynomial Hierarchy~\cite{EiterG95}. Consequently, Abduction is
harder than \NP and \coNP.  This complexity barrier rules out the
existence of a polynomial-time transformation to the propositional
satisfiability problem (\sat).  As a consequence, one cannot directly
apply powerful \sat solvers to Abduction.  However, this would be very
desirable in view of the enormous power of state-of-the-art \sat
solvers that can handle instances with millions of clauses and
variables~\cite{GomesKSS08,KatebiEtal11,JarvisaloEtal2012}.\smallskip

\parhead{Main contribution} We present a new approach to utilize
problem structure in order to break this complexity barrier for Abduction.
More precisely, we present transformations from Abduction to \sat that
run in quadratic time, with a constant factor that is exponential in
the distance of the given propositional theory from being \horn or
\krom. We measure the distance in terms of the size of a smallest
\emph{strong backdoor set} into the classes \horn or \krom,
respectively \cite{WilliamsGomesSelman03}.
Thus the exponential blow-up---which is to be expected when
transforming a problem from the second level of the Polynomial
Hierarchy into \sat---is confined to the size of the backdoor set,
whereas the order of the polynomial running time remains constant
independent of the distance. Such transformations are known as
\emph{fixed-parameter tractable reductions} and are fundamental to
Parameterized Complexity Theory~\cite{book/DowneyF99,book/FlumG06}. 

Our new approach to Abduction has several interesting aspects.  It
provides flexibility and openness. Any additional constraints
that one might want to impose on the solution to the Abduction problem
can simply be added as additional clauses to the \sat encoding. Hence
such constraints can be handled without the need for modifying the
basic transformation.  The reduction approach readily supports the
enumeration of subset-minimal solutions, as we can delegate the
enumeration to the employed \sat solver.  Any
progress made for \sat solvers directly translates into progress
for Abduction.

As a by-product, our approach gives rise to several new
fixed-parameter tractability results. For instance, Abduction
is fixed-parameter tractable when parameterized by the number of
hypotheses and the size of a smallest \horn- or \krom-backdoor
set. Parameterized by the size of backdoor sets alone, Abduction
is \paraNP-complete.

\parhead{Related Work} Methods from parameterized complexity have
turned out to be well-suited to tackle hard problems in Knowledge Representation \&
Reasoning~\cite{GottlobS08}.  In particular, the concept of backdoor
sets provides a natural way of parameterizing such
problems~\cite{SamerS09,FichteS11,GaspersS12,GaspersSzeider12,LacknerP12-closedworld,LacknerP12-minimalmodels,DvorakOrdyniakSzeider12}.
The parameterized complexity of Abduction was subject of previous
work, where different parameters have been considered
\cite{GottlobPW10,FellowsPRR12}. The most significant difference to
our work is that we use fixed-parameter tractability not to \emph{solve} the
Abduction problem itself, but to \emph{reduce} it from the second level of
the Polynomial Hierarchy to the first. This way, our parameters can be
less restrictive and are potentially small for larger classes of
Abduction instances.
This novel use of fixed-parameter tractability was recently applied in
the domain of answer-set programming~\cite{FichteS13}, and we believe
it can be 
applied to other hard reasoning problems as well.  Abduction can be
transformed in polynomial time to \qbf~\cite{EglyETW00}, but then one
remains on the second level of the Polynomial Hierarchy.

\section{Preliminaries}
\parhead{Propositional Logic}
A formula in conjunctive normal form is a conjunction of disjunctions
of literals; we denote the class of all such formulas by \cnf.  It is
convenient to view a formula in \cnf\ also as a set of clauses and a
clause as a set of literals.  \krom\ denotes the class of all \cnf\
formulas having clause size at most~$2$.  \horn\ formulas are \cnf\
formulas with at most one positive literal per clause.
Let $\var{\varphi}$ be the set of variables occurring in $\varphi\in\cnf$.

A (partial) \emph{truth assignment} is a mapping $\tau:X \rightarrow
\{0,1\}$ defined for a set $X$ of variables.  We write $\var{\tau}$ to
denote the domain $X$.  To extend $\tau$ to literals we put $\tau(\neg
x)=1 - \tau(x)$ for $x\in X$.  By $\ta{X}$ we denote the set of all
truth assignments $\tau: X\rightarrow \{0,1\}$.  Let $S$ be a set of
variables.  We denote by $\ta{X,S}$ the set $\{\tau \in \ta{X} \mid
\forall s\in S \cap X: \tau(s)=1\}$.  We say that $\tau$ satisfies
literal $l$ if $\tau(l) = 1$.  A clause is tautological if it contains
a variable $x$ and its negation $\neg x$.  A truth assignment $\tau$
\emph{satisfies} a \cnf\ formula if in each non-tautological clause,
there exist a literal that is satisfied  by $\tau$.
A \cnf\ formula $\varphi$ is \emph{satisfiable} (or \emph{consistent}) if there exists some truth assignment $\tau$ that satisfies $\varphi$.
If, additionally, $\var{\tau}$ contains all variables of $\varphi$, we call it a \emph{model} of $\varphi$.
The \emph{truth assignment reduct} of a \cnf\ formula
$\varphi$ under a truth assignment $\tau$ is the \cnf\
formula $\varphi[\tau]$ obtained from $\varphi$ by deleting all clauses
that are satisfied by~$\tau$ and by deleting from the remaining
clauses all literals that are set to $0$ by~$\tau$. 
Let $\varphi,\psi\in\cnf$, then $\varphi$ entails $\psi$ (denoted by $\varphi\vDash\psi$) if every model $\tau$ of $\varphi$ with $\var{\psi}\subseteq \var{\tau}$ is also a model of $\psi$.

Let $\Res(\varphi)$ denote the set of clauses computed from $\varphi\in\cnf$ by iteratively applying resolution and dropping tautological clauses until a fixed-point is reached.
Applying resolution adds the clause $C\cup D$ to $\varphi$ if $C\cup\{x\}\in \varphi$ and $D\cup\{\neg x\}\in\varphi$.
Let $C$ be a non-tautological clause then $C\in\Res(\varphi)$ if and only if $\varphi\vDash \{C\}$.
For more information, cf.~\cite{Leitsch97}.

Let $X$ be a set of variables, we define $\overline{X} \coloneqq \{\neg x\mid x\in X\}$.
\longversion{Furthermore, let $\tau_1$ and $\tau_2$ be truth assignments and $X$ a set of variables.
Then we denote by $\tau_1\sqsubseteq\tau_2$ that $\var{\tau_1}\subseteq\var{\tau_2}$ and by $\tau_1\sqcap X$ the restriction of the assignment $\tau$ to the domain $X$.}
For convenience, we view a set of variables $X$ as the conjunction of its elements, whenever it is used as a formula.\smallskip

\parhead{Propositional Abduction}
A \emph{(propositional) abduction instance} consists of a tuple $\langle V, H, M, T\rangle$, where $V$ is the set of \emph{variables}, $H \subseteq V$ is the set of \emph{hypotheses}, $M \subseteq V$ is the set of \emph{manifestations}, and $T$ is the \emph{theory}, a formula in \cnf over $V$. It is required that $M\cap H=\emptyset$.
We define the size of an abduction instance $\cP$ to be the size of a reasonable encoding of $\cP$.
For instance taking $\card{V} + \card{H} + \card{M} + \sum_{C \in T} \card{C}$ would do.
A set $S \subseteq H$ is a \emph{solution} to (or explanation of)  $\cP$ if $T \cup S$ is consistent and $T \cup S \models M$ (entailment).
$\sol(\cP)$ denotes the set of all solutions to $\cP$.
The \emph{solvability problem for propositional abduction} \abd is the following problem:

\begin{center}
\shortversion{\renewcommand{\tabcolsep}{4px}}
\shortversion{\begin{tabular}{|r p{0.75\columnwidth}|}}
\longversion{\begin{tabular}{|r p{0.52\columnwidth}|}}
\hline
\multicolumn{2}{|l|}{\abd} \\
\textit{Instance:} & An abduction instance $\cP$. \\
\textit{Problem:} & Decide whether $\sol(\cP) \neq \emptyset$.\\
\hline
\end{tabular}
\end{center}\longversion{\medskip}

\shortversion{\noindent}
We will additionally consider a version of the abduction problem where
we search for certain solutions only.  A solution $S$ is
\emph{subset-minimal} if there is no solution $S'\subsetneq S$.\longversion{\medskip}
\begin{center}
\shortversion{\renewcommand{\tabcolsep}{4px}}
\shortversion{\begin{tabular}{|rp{0.75\columnwidth}@{\hskip 6px}|}}
\longversion{\begin{tabular}{|rp{0.52\columnwidth}@{\hskip 6px}|}}
\hline
\multicolumn{2}{|l|}{\abdmin} \\
\textit{Instance:} & An abduction instance $\cP$ and a hypothesis $h$. \\ %
\textit{Problem:} & Is there a \emph{subset-minimal} solution $S\subseteq H$ to $\cP$ with $h\in S$.\\
\hline
\end{tabular}
\end{center}\longversion{\medskip}

Recall that \abd is \SigmaTwo-complete in general while it becomes \NP-complete when the theory is a \horn or \krom formula.
The same holds for \abdmin~\cite{EiterG95,NordhZ08,SelmanL90}.\smallskip

\parhead{Example} %
In this example we look for explanations why the nicely planned skiing trip ended up so badly.
The abduction instance is given as follows:
\begin{align*}
  V &= \{ \mathtt{snows}, \mathtt{rains}, \mathtt{precipitation}, \mathtt{warm}, \mathtt{hurt}, \mathtt{sad} \}\\
  H &= \{ \mathtt{precipitation}, \mathtt{warm}, \mathtt{hurt} \}\\
  M &= \{  \mathtt{sad} \}\\
  T &= \{\begin{aligned}[t]
    & \mathtt{precipitation}\ra\mathtt{rains}\vee\mathtt{snows}, \\
    & \mathtt{hurt}\ra \mathtt{sad}, \mathtt{warm}\ra\neg\mathtt{snows}, \mathtt{rains}\ra\mathtt{sad} \}
  \end{aligned}
\end{align*}
One can easily verify that $S_1=\{\mathtt{hurt}\}$ and $S_2=\{\mathtt{precipitation},\mathtt{warm}\}$ are solutions to this abduction instance.
Notice that $S_1$ and $S_2$ are the only subset-minimal solutions while, e.g., $S_3=\{\mathtt{hurt},\mathtt{warm}\}$ is a solution as well.\smallskip

\parhead{Parameterized Complexity}
We give some basic background on
parameterized complexity. For more detailed information we refer to
other sources~\cite{book/DowneyF99,book/FlumG06}.  A \emph{parameterized
  problem} $L$ is a subset of $\Sigma^* \times \Nat$ for some finite
alphabet $\Sigma$. For an instance $(I,k) \in \Sigma^* \times \Nat$ we
call $I$ the \emph{main part} and $k$ the \emph{parameter}. $L$ is
\emph{fixed-parameter tractable} if there exists a computable function
$f$ and a constant~$c$ such that there exists an algorithm that decides
whether $(I,k)\in L$ in time $\bigO(f(k)\CCard{I}^c)$ where $\CCard{I}$
denotes the size of~$I$.  Such an algorithm is called an
\emph{fpt-algorithm}.
We will use the $\bigO^*(\cdot)$ notation which is defined in the same way as $\bigO(\cdot)$, but ignores polynomial factors.

Let $L \subseteq \Sigma^* \times \Nat$ and $L' \subseteq \Sigma'^*\times
\Nat$ be two parameterized problems for some finite alphabets $\Sigma$
and $\Sigma'$. An \emph{fpt-reduction} $r$ from $L$ to $L'$ is a
many-to-one reduction from $\Sigma^*\times \Nat$ to $\Sigma'^*\times
\Nat$ such that for all $I \in \Sigma^*$ we have $(I,k) \in L$ if and
only if $r(I,k)=(I',k')\in L'$.
Thereby, $k' \leq g(k)$ for a fixed
computable function $g: \Nat \rightarrow \Nat$, and there is a computable
function $f$ and a constant $c$ such that $r$ is computable in time
$\bigO(f(k)\CCard{I}^c)$ where $\CCard{I}$ denotes the size
of~$I$~\cite{book/FlumG06}. Thus, an fpt-reduction is, in particular, an
fpt-algorithm.
We would like to note that the theory of fixed-parameter
intractability is based on fpt-reductions.\smallskip

\parhead{Backdoors}
\citex{WilliamsGomesSelman03} introduced the notion of \emph{backdoors}
to explain favorable running times and the heavy-tailed behavior of \sat
and CSP solvers on practical instances.  Backdoors are defined with
respect to a fixed class $\cC$ of \cnf formulas, the \emph{base class}
(or more figuratively, \emph{island of tractability}).
 A \emph{strong} \emph{$\cC$-backdoor set} of a formula $\varphi\in\cnf$ is
a set $\cB$ of variables such that $\varphi[\tau]\in \cC$ for each $\tau \in \ta{\cB}$.
$\cB$ is also called a strong $\cC$-backdoor set of an abduction
instance $\cP=\abdinst{V,H,M,T}$ if $\cB$ is a strong $\cC$-backdoor
set of~$T$.  Observe that the instance from the example above has a
strong \horn-backdoor set and a strong \krom-backdoor set of
cardinality one (consider, e.g., $\cB=\{\mathtt{snows}\})$.

\parhead{Backdoor Approach}
The backdoor approach consists of two phases.
First, a backdoor set is computed (\emph{detection}) and afterwards the backdoor is used to solve the problem (\emph{evaluation}).
For example, for \sat this approach works as follows.
If we know a strong $\cC$-backdoor set of a \cnf\ formula $\varphi$ of size $k$, we can
reduce the satisfiability of $\varphi$ to the satisfiability of $2^k$ easy
formulas that belong to the base class. The challenging problem,
however, is to find a strong backdoor set of size at most $k$, if it
exists. This problem is \NP-hard for all reasonable base classes, but
fortunately, fixed-parameter tractable for the base classes \krom and
\horn if parameterized by~$k$~\cite{NishimuraRagdeSzeider04-informal}. 
In particular, efficient fixed-parameter algorithms for the
\textsc{3-Hitting Set} problem and for the \textsc{Vertex Cover} problem
can be used for detecting strong backdoor sets for the base classes
\krom and \horn. Fastest known fixed-parameter algorithms for these two
problems run in time $\bigO^*(2.270^k)$ and $\bigO^*(1.2738^k)$
\cite{NiedermeierRossmanith03,ChenKanjXia10}, respectively.
For further information on the parameterized complexity of backdoor set detection we refer to a recent survey~\cite{GaspersSzeider12}.

For Abduction the detection phase is the same as for \sat, but the
evaluation phase becomes the new challenge. We therefore focus on the
evaluation phase, and assume that the backdoor set is provided as part
of the input. However, whether or not we provide the backdoor set as
part of the input does not affect the parameterized complexity of the
overall problem, since, as explained above, the detection of strong
\horn/\krom-backdoors is fixed-parameter tractable.

\section{Transformations using Horn Backdoors}
\label{sec:horn}

In this section we present the transformation from \abd to \sat using strong \horn-backdoor sets.
In our transformation we will build upon ideas from~\citex{DowlingG84} for computing the unique minimal model (with respect to set-inclusion) of \horn formulas in linear time.
Recall that manifestations in the considered Abduction formalism are
assumed to be positive literals (i.e., variables).
The following lemma captures the evaluation phase of our backdoor approach.
\shortversion{The proof can be found in the full version~\cite{arxiv-version}.}

\begin{lemma}
  \label{lem:bds-solv}
  Let $\cP=\abdinst{V,H,M,T}$ be an abduction instance, let $S \subseteq H$, and let $\cB\subseteq V$. %
  Then $S$ is a solution to $\cP$ if and only if
  \longversion{\begin{enumerate}[(i)]}
  \shortversion{\begin{inparaenum}[(i)]}
  \item $\exists \tau \in \ta{\cB,S}$ such that $T[\tau] \cup S$ is consistent, and
  \item $\forall \tau \in \ta{\cB,S}$, $T[\tau] \cup S\vDash M[\tau]$.
  \shortversion{\end{inparaenum}}
  \longversion{\end{enumerate}}
\end{lemma}

\longversion{\begin{proof}
  We start with the ``$\Rightarrow$'' direction.
  Assume that that $S$ is a solution to $\cP$.
  Therefore, $T\cup S$ must be consistent and $T\cup S\vDash M$ must hold.

  We first show that there is a $\tau \in \ta{\cB,S}$ such that $S\cup T[\tau]$ is consistent.
  Let $\tau_V$ be an assignment under which $T\cup S$ evaluates to \true.
  Such an assignment can be found, since $T\cup S$ is consistent.
  Notice that each $h\in S$ must be set to true by $\tau_V$ in order to satisfy $T\cup S$.
  We define $\tau$ to be the assignment $\tau_V$ restricted to variables in $\cB$.
  Hence it must be that case that $\tau \in \ta{\cB,S}$.
  Now, since $\tau_V$ is a model of $T\cup S$ and $\tau\sqsubseteq \tau_V$, $S\cup T[\tau]$ is satisfiable.

  Next, we show that $\forall \tau \in \ta{\cB,S}$, $S\cup T[\tau]\vDash M[\tau]$.
  Assume towards a contradiction that there is a $\tau \in \ta{\cB,S}$ such that $S\cup T[\tau]\nvDash M[\tau]$.
  This $\tau$ must set all $h\in S$ to true, since otherwise it would not be contained in $\ta{\cB,S}$.
  From $S\cup T[\tau]\nvDash M[\tau]$ we know that there is an assignment $\tau'$ that satisfies $\varphi\colonequals S\cup (T \cup \overline{M})[\tau]$. %
  The assignment $\tau'$ must set all variables in $S$ to true, since otherwise it could not satisfy the subformula $S$ in $\varphi$.
  Observe that $\var{\tau}\cap\var{\tau'}\subseteq S$.
  Therefore, we can construct the combination of $\tau$ and $\tau'$ as follows:
  \[\tau^*(x)\colonequals\begin{cases}
    \tau(x) & \text{if } x\in (\var{\tau}\setminus S)\\
    \tau'(x) & \text{if } x\in (\var{\tau'}\setminus S)\\
    \text{true} & \text{otherwise (i.e., $x\in S$)} %
  \end{cases}\]
  It is easy to verify that $\tau^*$ is a model of $S\cup T\cup \{\overline{M}\}$.
  This is a contradiction to the assumption of $T\cup S\vDash M$ and $S$ being a solution.

  It remains to show the ``$\Leftarrow$'' direction.
  Suppose that both condition (i) and (ii) are satisfied by $S$.
  We need to show that $S$ is a solution to $\cP$.

  As condition (i) is fulfilled we know that there is an assignment $\tau \in \ta{\cB,S}$ such that there is an assignment $\tau'$ satisfying $S\cup T[\tau]$.
  Let the assignment $\tau^*$ be defined as above.
  Remember that all $h\in S$ are set to true in $\tau^*$.
  Then, $\tau^*$ must be a satisfying assignment of $S\cup T$ and hence $S\cup T$ is consistent.
  This is because $S$ is trivially fulfilled and because $T$ must be satisfied by $\tau^*$.
  Otherwise $\tau'$ would not be a model of $S\cup T[\tau]$.

  As the last step we show that $T\cup S\vDash M$ is indeed fulfilled.
  We show this by contradiction.
  Assume that $S \cup T\nvDash M$.
  In other words $S \cup T \cup \overline{M}$ is satisfiable by an assignment $\tau$.
  Let now $\tau_1\colonequals \tau \sqcap (\cB\cup S)$ and $\tau_2\colonequals \tau \sqcap V\setminus (\cB\setminus S)$.
  Then $\tau_2$ is also a model of $S\cup (T \cup \overline{M})[\tau_1]$, which is a contradiction the assumption of condition (ii) being fulfilled.
\end{proof}
}

Based on this lemma, Algorithm~\ref{alg:horn-abd-bds-checker} checks whether a given candidate is indeed a solution.

{%
\shortversion{\SetAlFnt{\small}}
\SetKwFor{OwnRepeat}{repeat}{times}{end}
\SetKwInOut{Input}{Input}
\SetKwInOut{Output}{Output}
\SetKw{Break}{break}
\begin{algorithm}[ht]
  \DontPrintSemicolon
  \Input{An abduction instance $\cP=\abdinst{V,H,M,T}$, a strong \horn-backdoor set $\cB$ of $\cP$ and a solution candidate $S\subseteq H$ to be checked.}
  \Output{Decision whether $S$ is a solution to $\cP$.}
  \BlankLine
  consistent$\gets$ \textit{false}\;
  entailment $\gets$ \textit{true}\;

  \ForEach{$\tau \in \ta{\cB,S}$}
  {
    \If{$\neg$ \textnormal{consistent}}
    {
      \If{$T[\tau] \cup S$ is consistent}
      {
        consistent $\gets$ \textit{true}\;
      }
    }
    \If{$T[\tau] \cup S$ is consistent}
    {%
      U$\gets$ the unique minimal model of $T[\tau]\cup S$\;
      \If(\tcp*[f]{Thus $T[\tau] \cup S \nvDash M[\tau]$}\longversion{\hspace{3cm}}){$U\nvDash
        M[\tau]$}
      {
        entailment $\gets$ \textit{false}\;
        \Break
      }
    }
  }
  \Return{\textnormal{consistent} $\wedge$ \textnormal{entailment}}

\caption{Solution-Checker$_\text{\,\horn-bds}$}
\label{alg:horn-abd-bds-checker}
\end{algorithm}
} %

\begin{lemma}\label{lem:horn-alg-corr}
Let $\cP=\abdinst{V,H,M,T}$ be an abduction instance and $\cB$  a strong \horn-backdoor set of $\cP$. A set $S \subseteq H$ is a solution to $\cP$ if and only if Algorithm~\ref{alg:horn-abd-bds-checker} returns yes.
\end{lemma}

\begin{proof}
  $(\Rightarrow)$
  Assume there exists a solution $S \in \sol(\cP)$. By Lemma~\ref{lem:bds-solv} there exists a $\tau \in \ta{\cB,S}$ such that $T[\tau]\cup S$ is consistent.
  Therefore, for one of the assignments from Line~3, the flag \textit{consistent} will be set to \true in Line~6.
  Furthermore, by Lemma~\ref{lem:bds-solv}, for all $\tau \in \ta{\cB,S}$ it holds that $T[\tau] \cup S\vDash M[\tau]$.
  Therefore, Line~10 will not be reached and the flag \textit{entailment} will remain \true.
  Hence, the algorithm returns \textit{yes}.
  
  $(\Leftarrow)$
  Assume that the algorithm returns \textit{yes}.
  Therefore, Line~10 is never reached and $U \vDash M[\tau]$ for all $\tau \in \ta{\cB,S}$.
  Since $M[\tau]$ is entailed in the minimal model and contains only positive literals, it is entailed in every model and $T[\tau]\cup S\vDash M[\tau]$ for all $\tau \in \ta{\cB,S}$.
  Furthermore, there exists a $\tau \in \ta{\cB,S}$ such that Line~6 is reached and hence $T[\tau]\cup S$ is consistent.
  It follows from Lemma~\ref{lem:bds-solv} that $S \in \sol(\cP)$.
\end{proof}

\begin{corollary}
  Let $\cP=\abdinst{V,H,M,T}$ be an abduction instance and $\cB$  a
  strong \horn-backdoor set of $\cP$.  We can check whether $\cP$ has
  a solution in time $\bigO^*(2^{\card{\cB}+\card{H}})$.  Hence, \abd
  is fixed-parameter tractable when parameterized by
  $\card{\cB}+\card{H}$.
\end{corollary}
\begin{proof}
  One has to check for each of the $2^{\card{H}}$-many solution candidates $S\subseteq H$, whether $S$ is a solution to~$\cP$.
  To this end we apply Algorithm~\ref{alg:horn-abd-bds-checker}, which runs in time $\bigO^*(2^{\card{\cB}})$.
\end{proof}

If the number of hypotheses is large, this fpt-algorithm is not
efficient.  To overcome this limitation, we present next an
fpt-reduction to \sat using only the backdoor size  as the parameter.
This is the main result of this section.

\begin{theorem}
\label{thm:horn-bds-solv}
Given an abduction instance $\cP$ of input size $n$ and a strong \horn-backdoor
$\cB$ of $\cP$ of size~$k$, we can create in time $\bigO(2^k n^2)$ a \cnf formula $F_\text{\horn-Solv}$ of size $\bigO(2^k n^2)$ such that $F_\text{\horn-Solv}$ is satisfiable if and only if $\sol(\cP) \neq \emptyset$.
\end{theorem}

\begin{proof}
Let $\cP=\abdinst{V,H,M,T}$.
We will first construct a propositional formula $F_\text{\horn-Solv}'$ which is not in \cnf.
The required \cnf formula $F_\text{\horn-Solv}$ can then be obtained from $F_\text{\horn-Solv}'$ by means of the well-known transformation due to~\citex{Tseitin68}, in which auxiliary variables are introduced.
This transformation produces for a given propositional formula $F'$ in linear time a \cnf formula $F$ such that both formulas are equivalent with respect to their satisfiability, and the length of $F$ is linear in the length of~$F'$.

Note that in this encoding a solution $S \subseteq H$ can be obtained
by projecting a model of $F_\text{\horn-Solv}'$ to the variables in $H$, i.e., $h\in S$ if and only if $h \in H$ is true in the model.
We define
\[
  F_\text{\horn-Solv}' \coloneqq T \wedge F^\text{ent},
\]
where $F^\text{ent}$ is a formula, defined below, that checks the entailment $T \cup S \models M$.
Let $\cB_1,\dots,\cB_{2^k}$ be an enumeration of all the subsets of $\cB$.
Each subset implicitly denotes a truth assignment for $\cB$.
For each variable $v \in \cB$ and each subset $\cB_i$ with $1 \leq i \leq 2^k$, we use a propositional constant $B_i(v)$ that is true if and only if $v \in \cB_i$.
Now we can define $F^\text{ent}$:
\[
  F^\text{ent} \coloneqq \bigwedge_{1\leq i \leq 2^k} \Bigl(\bigl(\bigwedge_{h\in H\cap \cB}(h \rightarrow B_i(h))\bigr) \rightarrow F_i^\text{ent}\Bigr),
\]
where $F_i^\text{ent}$, defined below, checks entailment for the $i$-th truth assignment for $\cB$.
By Lemma~\ref{lem:bds-solv} we have to ensure entailment only for those truth assignments that match the truth value of all $h \in S$.
This is done via the implication $\bigl(\bigwedge_{h\in H\cap \cB}(h \rightarrow B_i(h))\bigr) \rightarrow F_i^\text{ent}$.
We define $F_i^\text{ent}$ using three auxiliary formulas:
\begin{align}\label{eq:krom-ent}
  F_i^\text{ent} \coloneqq (F_i^\text{lm} \wedge F_i^\text{check}) \rightarrow F_i^\text{man},
\end{align}
where $F_i^\text{lm}$ creates the least model of the \horn theory, $F_i^\text{check}$ checks whether this model satisfies all constraints of the theory, and $F_i^\text{man}$ checks if the model satisfies all manifestations.
Next we will define $F_i^\text{lm}$.
The idea behind the construction is to simulate the linear-time algorithm of \citex{DowlingG84}, where initially all variables are set to false and then a variable is flipped from false to true if and only if it is in the head of a rule where all
the variables in the rule body are true (a fact is a rule with empty body).
Once a fixed-point is reached, we have obtained the least model.
We encode this idea as follows.
Since we are interested in the least model of $T \cup S$ instead of just $T$, we initialize those variables that are contained in $S$ by setting them to true.
Let $p \coloneqq \min\{\card{T},\card{V}\}$ be the maximum number of steps after which the fixed-point is reached.
For each $i$ with $1 \leq i \leq 2^k$, we introduce a new set of variables $U_i \coloneqq \{ u_i^j[v] \mid v \in V, 0 \leq j \leq p \}$.
The intended meaning of a variable $u_i^j[v]$ is the truth value of the original variable $v$ after the $j$-th step of the computation of the least model.
The following auxiliary formulas encode this computation:
\begin{align*}
  F_i^\text{lm} &\coloneqq \bigwedge_{v \in V, 0\leq j \leq p} F_i^{(v,j)}; \\
  F_i^{(v,0)} &\coloneqq
    \begin{cases}
      u_i^0[v] \leftrightarrow h & \text{if } v=h \in H,\\
      u_i^0[v] \leftrightarrow \text{false} & \text{otherwise};
    \end{cases} \\
\shortversion{F_i^{(v,j)} &\coloneqq \begin{aligned}[t]&u_i^j[v] \leftrightarrow \Bigl(u_i^{j-1}[v] \vee \\
  &\bigvee_{\mathclap{\substack{r \in \text{Rules}(T[\cB_i]),\\v=\text{Head}(r)}}}\qquad \qquad  \bigwedge_{\mathclap{b \in \text{Body}(r)}} u_i^{j-1}[b]\Bigr) \quad (\text{for }1\leq j \leq p).\end{aligned}}
\longversion{F_i^{(v,j)} &\coloneqq u_i^j[v] \leftrightarrow \Bigl(u_i^{j-1}[v] \vee \bigvee_{\substack{r \in \text{Rules}(T[\cB_i]),\\v=\text{Head}(r)}} \ \ \bigwedge_{b \in \text{Body}(r)} u_i^{j-1}[b]\Bigr) \quad (\text{for }1\leq j \leq p).}
\end{align*}
As mentioned above, we initially set the variables to false with the exception of the hypotheses.
This is done in the formulas $F_i^{(v,0)}$.
The computation steps are represented by the formulas $F_i^{(v,j)}$.
Thereby we set a variable $u_i^j[v]$ to true if and only if it was already true in the previous step ($u_i^{j-1}[v]$), or there is a rule $r$ in $T[B_i]$ such that $\text{Head}(r)=v$ and all body variables $b \in \text{Body}(r)$ were already true in the previous step ($u_i^{j-1}[b]$).
In order to check whether the least model satisfies all the constraints (purely negative clauses), we define:
\[
  F_i^\text{check} \coloneqq \bigwedge_{C \in \text{Constr}(T[\cB_i])} \shortversion{\ }\longversion{\ \ } \bigvee_{v \in C} \neg u_i^p[v].
\]
Finally, we check whether the model satisfies all manifestations with the following formula:
\[
  F_i^\text{man} \coloneqq \bigwedge_{m \in M \setminus \cB} u_i^p[m] \wedge \bigwedge_{m \in M \cap \cB} B_i(m).
\]
It follows by Lemma~\ref{lem:horn-alg-corr} and by the construction of the auxiliary formulas that $F_\text{\horn-Solv}'$ is satisfiable if and only if $\sol(\cP) \neq \emptyset$.
Hence it remains to observe that for each $1 \leq i\leq 2^k$ the auxiliary formula $F_i^\text{lm}$ can be constructed in quadratic time, whereas the auxiliary formulas $F_i^\text{check}$ and $F_i^\text{man}$ can be constructed in linear time.
Therefore, the formula $F^\text{ent}$ can be constructed in time $\bigO(2^k n^2)$ and has size $\bigO(2^k n^2)$.
\end{proof}

\section{Transformations using Krom Backdoors}
\label{sec:krom}

Recall that (in contrast to \horn formulas) a \krom formula might have
several (subset) minimal models.  Hence we cannot use the above
approach for the base class \krom.  However, we can exploit special
properties of \krom formulas with respect to resolution.  Analogously
to Section~\ref{sec:horn} we start with an algorithm for verifying a
solution and show fixed-parameter tractability with respect to the
combined parameter backdoor size and number of
hypotheses. Subsequently we establish the main result of this section,
an fpt-reduction with backdoor size as the single parameter.

{%
\shortversion{\SetAlFnt{\small}}
\SetKwFor{OwnRepeat}{repeat}{times}{end}
\SetKwInOut{Input}{Input}
\SetKwInOut{Output}{Output}
\SetKw{Break}{break}
\begin{algorithm}[ht]
  \DontPrintSemicolon
  \Input{An abduction instance $\cP=\abdinst{V,H,M,T}$, a strong \krom-backdoor set $\cB$ of $\cP$ and a solution candidate $S\subseteq H$ to be checked.}
  \Output{Decision whether $S$ is a solution to $\cP$.}
  \BlankLine
  consistent$\gets$ \textit{false}\;
  entailment $\gets$ \textit{true}\;

  \ForEach{$\tau \in \ta{\cB,S}$}
  {
    \If{$\neg$ \textnormal{consistent}}
    {
      \If{$T[\tau] \cup S$ is consistent}
      {
        consistent $\gets$ \textit{true}\;
      }
    }
    \eIf{$M[\tau]$ is consistent} %
    {
      \ForEach{$m\in M\setminus\var{\tau}$}
      {
        \If{$(T[\tau]\nvDash m) \wedge (\forall h\in S\!\!: T[\tau]\wedge h\nvDash m)$}
        {
          entailment $\gets$ \textit{false}\;
          \Break
        }
      }
    }
    {
      \If{$T[\tau]\cup S$ is consistent}
      {
        entailment $\gets$ \textit{false}\;
        \Break
      }
    }
  }
  \Return{\textnormal{consistent} $\wedge$ \textnormal{entailment}}
\caption{Solution-Checker$_\text{\,\krom-bds}$}
\label{alg:krom-abd-bds-checker}
\end{algorithm}%
} %
\begin{lemma}
Let $\cP=\abdinst{V,H,M,T}$ be an abduction instance and $\cB$ be a strong \krom-backdoor set of $\cP$. A set $S \subseteq H$ is a solution to $\cP$ if and only if Algorithm~\ref{alg:krom-abd-bds-checker} returns yes.
\end{lemma}

\longversion{\begin{proof}
  $(\Rightarrow)$
  Assume that there is a solution $S \in \sol(\cP)$.
  By Lemma~\ref{lem:bds-solv} there exists a $\tau \in \ta{\cB,S}$ such that $T[\tau] \cup S$ is consistent.
  Hence, for one of the assignments in Line~3, the variable \textit{consistent} will be set to \true in Line~6.
  Furthermore, by Lemma~\ref{lem:bds-solv}, for all $\tau \in \ta{\cB,S}$ it holds that $T[\tau] \cup S\vDash M[\tau]$.
  We distinguish between two cases.

  Case (i): $M[\tau]$ is satisfiable.
  In this case the manifestations in $M\setminus\var{\tau}$ remain to be checked.
  These manifestations are checked one by one in Line~8.
  Since each manifestation $m\in M\setminus\var{\tau}$ is entailed by $T[\tau] \cup S$ and $T[\tau]$ is a \krom formula, one can show that either $T[\tau]\vDash m$ or there is a $h\in S$ such that $T[\tau]\wedge h\vDash m$.
  Therefore, the variable \textit{entailment} is never set to \false in Line~10.
  
  Case (ii): $M[\tau]$ is not satisfiable.
  Then $T[\tau]\cup S$ cannot be satisfiable since this would contradict the assumption that $T[\tau]\cup S\vDash M[\tau]$ for all $\tau \in \ta{\cB,S}$.
  As a consequence, the variable \textit{entailment} cannot be set to \false in Line~14.

  Taken together, the algorithm returns \textit{yes}.
  
  $(\Leftarrow)$
  Assume that the algorithm returns \textit{yes}.
  This can only be the case if neither Line~10 nor Line~14 are reached.
  Therefore, if there is a $\tau$ such that $M[\tau]$ is satisfiable for each manifestation $m\in M\setminus\var{\tau}$, either $T[\tau]\vDash m$ must hold or there is a $h\in S$ such that $T[\tau]\wedge h\vDash m$.
  Conversely, if $M[\tau]$ is not satisfiable, $T[\tau]\cup S$ is not satisfiable either.
  In both cases this ensures that $T[\tau]\cup S\vDash M[\tau]$ indeed holds.
  The loop comprising Lines~3-15 makes sure that $T[\tau]\cup S\vDash M[\tau]$ for each $\tau \in \ta{B,S}$.
  In addition, there exists a $\tau \in \ta{\cB,S}$ such that Line~6 is reached and hence $T[\tau]\cup S$ must be consistent.
  It follows from Lemma~\ref{lem:bds-solv} that $S \in \sol(\cP)$.
\end{proof}}

\begin{corollary}
  Let $\cP=\abdinst{V,H,M,T}$ be an abduction instance and $\cB$ be a strong \krom-backdoor set of $\cP$.
  We can check whether $\cP$ has a solution in time $\bigO^*(2^{\card{\cB}+\card{H}})$.
  Hence, \abd is fixed-parameter tractable when parameterized by $\card{\cB}+\card{H}$.
\end{corollary}
The transformation in the next theorem will use resolution in a preprocessing step, where we compute for each assignment $\tau\in\ta{\cB,S}$ the set containing all (non-tautological) resolvents restricted to variables in $H\cup M$.
Computing these resolvents is possible in polynomial time for \krom formulas, since resolution on a \krom\ formula always yields a \krom\ formula.
We employ the following definition and lemma.
\begin{definition}[\citex{FellowsPRR12}, Definition 17]
  \label{def:trimres}
  Given an abduction instance for \krom\ theories $\langle V, H, M, T \rangle$.
  We define the function $\TrimRes(T,H,M) \colonequals \{ C \in \Res(T) \mid C \subseteq X\}$, with $ X = H \cup M \cup \{\neg x \mid x \in (H \cup M)\}$.
  In case $\TrimRes(T,H,M)$ contains the empty clause $\Box$, we set $\TrimRes(T,H,M) \colonequals \{\Box\}$.
\end{definition}
\begin{lemma}[\citex{FellowsPRR12}, Lemma 20]
  \label{lem:krom-prep-clauses}
  Let $\langle V, H, M, T \rangle$ be an abduction instance for \krom\ theories, $S \subseteq H$, $m \in M$, and $T \wedge S$ be satisfiable.
  Then $T \wedge S \models m$ if and only if either $\{m\} \in \TrimRes(T,H,M)$ or there exists some $h \in S$ with $\{\neg h, m\} \in \TrimRes(T,H,M)$.
\end{lemma}
We extend this notion by using $\textit{TrimRes}(T,H,M,\tau)\coloneqq$ $\textit{TrimRes}(T[\tau],H\setminus\var{\tau},M\setminus\var{\tau})$.
\begin{theorem}
\label{thm:krom-bds-solv}
Given an abduction instance $\cP$ of input size $n$ and a strong \krom-backdoor
$\cB$ of $\cP$ of size $k$, we can create in time $\bigO(2^kn^2)$ a \cnf formula $F_\text{\krom-Solv}$ of size $\bigO(2^k n^2)$ such that $F_\text{\krom-Solv}$ is satisfiable if and only if $\sol(\cP) \neq \emptyset$.
\end{theorem}
\begin{proof}
By the same argument as in Theorem~\ref{thm:horn-bds-solv}, it suffices to create first a formula $F_\text{\krom-Solv}'$ which is not in \cnf.
Formula $F_\text{\krom-Solv}'$ is identical to formula $F_\text{\horn-Solv}'$ except for the subformula $F_i^\text{ent}$, for which a completely new approach is needed. %

Let $\cP=\abdinst{V,H,M,T}$ and let $\cB_1,\dots,\cB_{2^k}$ be an enumeration of all the subsets of $\cB$.
Each subset $\cB_i$ implicitly defines a truth assignment $\tau_i$ of $\cB$.
In a preprocessing step, for each assignment $\tau_i$ the function $\textit{TrimRes}(T,H,M,\tau_i)$ is computed.
In order, to connect the result of $\textit{TrimRes}(T,H,M,\tau_i)$ with the encoding we use logical constants which are defined as follows:
Let $C$ be a clause over $H\cup M$ or the empty clause $\Box$.
Then we define $\TR_i^C$ to be true if and only if $C\in\textit{TrimRes}(T,H,M,\tau_i)$.

We obtain the formula $F_\text{\krom-Solv}'$ from $F_\text{\horn-Solv}'$ by replacing $F_i^\text{ent}$, see Equation~\eqref{eq:krom-ent}, by:
\begin{align*}
\shortversion{F_i^\text{ent} &\coloneqq \Bigl(\;\bigwedge_{\mathclap{m\in M\cap \cB}}B_i(m)\ra \varphi_i^\text{ent}\;\Bigr) \wedge \Bigl(\;\bigvee_{\mathclap{m\in M\cap \cB}}\neg B_i(m)\ra \psi_i^\text{ent}\;\Bigr)\\}
\longversion{F_i^\text{ent} &\coloneqq \Bigl(\;\bigwedge_{m\in M\cap \cB}B_i(m)\ra \varphi_i^\text{ent}\;\Bigr) \wedge \Bigl(\;\bigvee_{m\in M\cap \cB}\neg B_i(m)\ra \psi_i^\text{ent}\;\Bigr)\\}
  \varphi_i^\text{ent} &\coloneqq \bigwedge_{m\in M\setminus \cB} \biggl( \TR_i^{\{m\}} \vee \bigvee_{h\in H} \left( h \wedge \TR_i^{\{h\ra m\}}\right)\biggr)\\
  \psi_i^\text{ent} &\coloneqq \TR_i^\Box\vee \bigvee_{h\in
    H}\left(h\wedge \TR_i^{\{\neg h\}}\right)\vee \mbox{} 
 \longversion{ \bigvee_{h_1,h_2\in H}\left(h_1\wedge h_2
     \wedge \TR_i^{\{\neg h_1,\neg h_2\}}\right) }
 \shortversion{\\ &\quad \bigvee_{h_1,h_2\in H}\left(h_1\wedge h_2
     \wedge \TR_i^{\{\neg h_1,\neg h_2\}}\right) }
\end{align*}
For each assignment $\tau_i$ (represented by $B_i(\cdot)$) we have to check whether the entailment $T[\tau_i]\cup S\vDash M[\tau_i]$ holds.
This is done in $F_i^\text{ent}$.
\shortversion{The proof of the correctness of this construction can found in the full version~\cite{arxiv-version}.}
\longversion{In the entailment check we need to distinguish between two cases.
The question is whether the assignment $\tau_i$ assigns false to some manifestation and thus ``disturbs'' the entailment.
In such a case the entailment can only be fulfilled if $T[\tau_i]\cup S$ is unsatisfiable.

Case (i): If $\bigwedge_{m\in M\cap \cB}B_i(m)$ is true, i.e., all manifestations contained in the backdoor set are set to true, it suffices to check whether for each manifestation $m\in M\setminus \cB$ either $\{m\}\in \textit{TrimRes}(T,H,M,\tau_i)$ or there is some $h\in S$ such that $\{h\ra m\}\in \textit{TrimRes}(T,H,M,\tau_i)$.
The correctness can be seen from Lemma~\ref{lem:krom-prep-clauses}.

Case (ii): If there is some $m\in M\cap \cB$ such that $B_i(m)$ is false, i.e., it is being set to false in the assignment $\tau_i$, we need to have a closer look at the unsatisfiability of $T[\tau_i] \cup S$.
There are three possibilities that can render  $T[\tau_i] \cup S$ unsatisfiable.
One can verify that in case $\bigvee_{m\in M\cap \cB}\neg B_i(m)$ is satisfiable the formula exactly checks these three conditions.
\begin{itemize}
\item $\Box\in\textit{TrimRes}(T,H,M,\tau_i)$, i.e., $T[\tau_i]$ is unsatisfiable.
\item For some $h$, which is set to true, $\{\neg h\}\in\textit{TrimRes}(T,H,M,\tau_i)$, i.e., setting $h$ to true is not consistent with $T[\tau_i]$.
\item For some pair of hypotheses $h_1,h_2\in H$, which are both set to true, $\{\neg h_1,\neg h_2\}\in\textit{TrimRes}(T,H,M,\tau_i)$, i.e., setting $h_1$ and $h_2$ to true is not consistent with $T[\tau_i]$.
\end{itemize}}
Note that $F_i^\text{ent}$ can be constructed in quadratic time.
Thus $F_\text{\krom-Solv}$ can be constructed in time and space of $\bigO(2^kn^2)$.
\end{proof}

\section{Subset Minimality}
\label{sec:subset}
In abductive reasoning one is often more interested in  subset-minimal
solutions than in ``ordinary'' solutions~\cite{EiterG95}.
Consider the example from the preliminaries.  Clearly, $S_3$ is a
solution but actually the hypothesis $\mathtt{warm}$ is dispensable.
In larger settings unnecessary hypotheses in the solution might blur
the explanation.  We demonstrate now how the previously presented
transformations can be modified to produce  only subset-minimal solutions
and thus solve \abdmin.

\begin{theorem}
\label{thm:bds-sub-solv}
Given an instance $\langle\cP,h^*\rangle$ for \abdmin of size $n$ and a strong \horn- or \krom-backdoor set $\cB$ of $\cP$ of size~$k$, we can create in time $\bigO(2^k n^2)$ a \cnf formula of size $\bigO(2^k n^2)$ that is satisfiable if and only if $\langle\cP,h^*\rangle$ is a yes-instance for \abdmin.
\end{theorem}

\begin{proof}
In order to construct a formula $F_\text{\horn-$\subsetneq$-Solv}$ in case of a strong \horn-backdoor set, we modify $F_\text{\horn-Solv}'$ from Theorem~\ref{thm:horn-bds-solv}.
The first change is that a solution $S \subseteq H$ is no longer represented by the restriction of a model to the variables in $H$.
Instead we introduce a new propositional variable $s_h$ for each $h \in H$ with the intended meaning of $s_h$ being true if and only if $h \in S$.
Let $V_S$ denote the set of these variables.
\[
  F_\text{\horn-Solv}' \coloneqq \bigwedge_{h \in H} (s_h \rightarrow h) \wedge T \wedge F^\text{ent}.
\]
The new conjunction $\bigwedge_{h \in H} (s_h \rightarrow h)$ ensures that the choice over the variables $V_S$ is propagated to the variables representing the hypotheses.
Furthermore, we change $F^\text{ent}$ and $F_i^{(v,0)}$ so that these formulas use the new variables $V_S$:
\begin{align*}
  F^\text{ent} & \coloneqq \bigwedge_{1\leq i \leq 2^k} \Bigl(\bigl(\bigwedge_{h\in H\cap \cB}(s_h \rightarrow B_i(h))\bigr) \rightarrow F_i^\text{ent}\Bigr);\\
  F_i^{(v,0)} & \coloneqq
    \begin{cases}
      u_i^0[v] \leftrightarrow s_h & \text{if } v=h \in H,\\
      u_i^0[v] \leftrightarrow \text{false} & \text{otherwise}.
    \end{cases}
\end{align*}
The other subformulas from the proof of Theorem~\ref{thm:horn-bds-solv} remain unchanged.
We can now write $F_\text{\horn-$\subsetneq$-Solv}'$ as follows.
\begin{align*}
  F_\text{\horn-$\subsetneq$-Solv}' & \coloneqq  s_{h^*} \wedge F_\text{\horn-Solv}' \wedge \bigwedge_{h \in H} (s_h \rightarrow F_\text{non-ent}^h);\\
  F_\text{non-ent}^h & \coloneqq \neg h^h \wedge \bigwedge_{v \in H\setminus \{h\}} (s_v \rightarrow v^h) \wedge T^h \wedge \overline{M}^h.
\end{align*}
where for each $h\in H$ the formula $F_\text{non-ent}^h$ enforces that
for $S\setminus \{h\}$ the manifestations $M$ are no longer entailed.
Thereby we introduce for each $h \in H$ a new copy $v^h$ of each
variable $v\in V$. Let $T^h$ and $\overline{M}^h$ denote $T$ and
$\overline{M}$, respectively, where the variables are replaced by the
new copies.
It remains to observe that the formula $F_\text{\horn-$\subsetneq$-Solv}'$ can be constructed in time $\bigO(2^k n^2)$.

The construction of the formula $F_\text{\krom-$\subsetneq$-Solv}$ in case of a strong \krom-backdoor set is analogous.
We modify the formula $F_\text{\krom-Solv}'$ from Theorem~\ref{thm:krom-bds-solv} similarly to the \horn case above.
Again we use new variables $s_h$ to decouple the solution from the hypotheses.
The important change is to replace the subformula $F_\text{\horn-Solv}'$ by a decoupled version of $F_\text{\krom-Solv}'$.
This decoupling is achieved by adding the clauses $\bigwedge_{h \in H} (s_h \rightarrow h)$ and replacing each occurrence of $h$ by $s_h$, and $h_1,h_2$ by $s_{h_1},s_{h_2}$ accordingly.
\end{proof}

\section{Completeness for \paraNP}

A parameterized problem $L$ is contained in the parameterized complexity class \paraNP if $L$ can be decided by a \emph{nondeterministic} fpt-algorithm~\cite{FlumG03}.

For a non-parameterized problem that is \NP-complete, it is considered a bad result if adding a parameter makes it \paraNP-complete, since this indicates that the considered parameter does not help.
But in the case of abduction, which is \SigmaTwo-complete, showing that it becomes \paraNP-complete is indeed a positive result.
In fact we get the following result as a corollary to Theorems~\ref{thm:horn-bds-solv}, \ref{thm:krom-bds-solv}, and \ref{thm:bds-sub-solv}.

\begin{corollary}
  For $\mathcal{C}\in \{\horn,\krom\}$, the problems \abd and \abdmin
  are \paraNP-complete when parameterized by the size of a smallest
  strong $\mathcal{C}$-backdoor set of the given abduction instance.
\end{corollary}

\section{Enumeration and Further Extensions}

In this section we sketch how the transformations presented above can
be used to \emph{enumerate} all (subset-minimal) solutions.

Obtaining a solution to the abduction instance from the models returned by the \sat solver is straightforward.
For the formulas $F_\text{\horn-$\subsetneq$-Solv}$ and $F_\text{\krom-$\subsetneq$-Solv}$ (Theorem~\ref{thm:bds-sub-solv}) it suffices to restrict the models to the variables $s_h$, for all $h\in H$.
In order to enumerate all possible solutions, one can exclude already found solutions by adding appropriate clauses that eliminate exactly these models.

The formulas $F_\text{\horn-Solv}$ and $F_\text{\krom-Solv}$ (Theorems~\ref{thm:horn-bds-solv} and~\ref{thm:krom-bds-solv}) need to be slightly modified, since a solution $S \subseteq H$ is not expressed explicitly.
This is for example a problem, if an hypothesis $h_1$ implies another hypothesis $h_2$, because no solution candidate containing only $h_1$ but not $h_2$ will be considered.
This problem does not occur in the formulas $F_\text{\horn-$\subsetneq$-Solv}'$ and $F_\text{\krom-$\subsetneq$-Solv}'$, since there the encoding of a solution is decoupled from the hypotheses via $\bigwedge_{h \in H} (s_h \rightarrow h)$.
The same technique can be used in the encoding of $F_\text{\horn-Solv}'$ and $F_\text{\krom-Solv}'$.
Note, that several occurrences of hypotheses $h$ have to be changed to $s_h$.

The transformations from Sections~\ref{sec:horn} and~\ref{sec:krom}
can be extended easily to the corresponding \emph{relevance problem},
i.e., asking whether there exists some $S\in \sol(\cP)$ containing a
certain $h^*\in H$.  It suffices to check only those $S\subseteq H$
where $h^*\in S$.
\begin{corollary}
  Let $\mathcal{C}\in \{\horn,\krom\}$.  Given an abduction instance
  $\cP=\abdinst{V,H,M,T}$, an atom $h^*\in H$, and a strong
  $\mathcal{C}$-backdoor set of~$\cP$, we can decide whether $h^*$
  belongs to some solution to $\cP$ in time
  $\bigO^*(2^{\card{\cB}+\card{H}})$.  Thus, the relevance problem is
  fixed-parameter tractable when parameterized by
  $\card{\cB}+\card{H}$.
  This also holds  for \abdmin.
\end{corollary}
Furthermore, the \sat encoding allows us to easily restrict solutions
for \abd or \abdmin in terms of any constraints that are expressible
in \cnf.  For example, using the encoding of a counter~\cite{Sinz05},
we can restrict the cardinalities of solutions and therefore solve the
variants of \abd as proposed by~\citex{FellowsPRR12}.  In contrast,
adding these constraints directly to the theory of the Abduction
instance can increase the size of the backdoor set.

\section{Conclusion}

We have presented fixed-parameter tractable transformations from
various kinds of abduction-based reasoning problems to \sat that
utilize small \horn/\krom-backdoor sets in the input. These
transformations are \emph{complexity barrier breaking reductions} as
they reduce problems from the second level of the Polynomial Hierarchy
to the first level.  A key feature of our transformations is that the
exponential blowup of the target \sat instance can be confined in
terms of the size of a smallest backdoor set of the input theory, a
number that measures the distance to the ``nice'' classes of \horn and
\krom formulas.
There are various possibilities for further reducing the size of the
target instance, which would be important for a practical
implementation.  For instance, one could use more sophisticated
computations of the least model combined with target languages that are
more compact than propositional
\cnf~\cite{Janhunen04,JanhunenNS09,ThiffaultBacchusWalsh04}.
An extension of our approach to Abduction with other notions of
solution-minimality, as surveyed by \citex{EiterG95}, is left for
future work.
Adding empty clause detection can lead to smaller backdoors~\cite{DilkinaGS07} and thus making our approach applicable to a larger class of instances.
While finding such backdoors is not fixed-parameter tractable~\cite{Szeider09}, one could use heuristics to compute them~\cite{DilkinaGS07}.

\shortversion{

}

\longversion{\bibliographystyle{named}
\bibliography{abd-bs}
}  
\end{document}